\documentclass[11pt]{article}

\usepackage{amsmath, amssymb, amsthm}
\usepackage{algorithm}
\usepackage{algorithmic}
\usepackage{graphicx}
\usepackage{dsfont}
\usepackage{booktabs}
\usepackage{multirow}  
\usepackage{hyperref}
\usepackage[numbers]{natbib}
\usepackage{geometry}
\usepackage{tikz}

\newtheorem{theorem}{Theorem}[section]
\newtheorem{lemma}[theorem]{Lemma}

\newtheorem{definition}[theorem]{Definition}

\newtheorem{example}{Example}

\newcommand{\R}{\mathbb{R}}

\title{Hierarchical Persistence Velocity for Network Anomaly Detection: Theory and Applications to Cryptocurrency Markets}

\author{
Omid Khormali\\
\small Department of Mathematics, University of Evansville, Evansville, Indiana 47722, USA\\
\small \texttt{ok16@evansville.edu}
}

\date{\today}

\begin{document}
	
	\maketitle
	
	\begin{abstract}
		We introduce the Overlap-Weighted Hierarchical Normalized Persistence Velocity (OW-HNPV), a novel topological data analysis method for detecting anomalies in time-varying networks. Unlike existing methods that measure cumulative topological presence, we introduce the first velocity-based perspective on persistence diagrams, measuring the rate at which features appear and disappear, automatically downweighting noise through overlap-based weighting. We also prove that OW-HNPV is mathematically stable. It behaves in a controlled, predictable way, even when comparing persistence diagrams from networks with different feature types. Applied to Ethereum transaction networks (May 2017–May 2018), OW-HNPV demonstrates superior performance for cryptocurrency anomaly detection, achieving up to 10.4\% AUC gain over baseline models for 7-day price movement predictions. Compared with established methods, including Vector of Averaged Bettis (VAB), persistence landscapes, and persistence images, velocity-based summaries excel at medium- to long-range forecasting (4–7 days), with OW-HNPV providing the most consistent and stable performance across prediction horizons. Our results show that modeling topological velocity is crucial for detecting structural anomalies in dynamic networks.
		
	\end{abstract}
	
	\section{Introduction}

    Blockchain networks such as Bitcoin and Ethereum generate massive amounts of publicly available transactional data, forming large and constantly changing graphs that reflect how digital assets move through the system \cite{Abay2019, Li2020}. Because these graphs capture real-time user activity, they offer a valuable window into broader market behavior, including sudden and unusual price changes \cite{Abay2019}. Ethereum, in particular, is widely used for smart contracts and thousands of user-created tokens, making its daily transaction graph rich in information and extremely complex because it is large, sparse, dynamic, and difficult to analyze with traditional graph-mining tools \cite{Li2020, Victor2020, AlphaCore2021}. These methods often struggle with the scale and irregularity of the network, even though detecting abnormal price movements is crucial for understanding major economic events and shifts in investor behavior \cite{Li2020}. To use blockchain graph data effectively for anomaly detection, the graph must first be transformed into stable and informative numerical features that can be fed into predictive models \cite{Abay2019, Li2020}. Developing such features is therefore essential for building reliable approaches to detecting unusual activity in cryptocurrency markets.\\

    Topological Data Analysis (TDA) has expanded rapidly over the past decade as researchers have sought new ways to understand the shape of complex data. Rather than examining only individual observations or simple pairwise relationships, TDA captures broader structural patterns such as connected components, loops, and voids that reveal how data are organized across many scales. This multiscale view has made TDA a powerful tool in many scientific domains, and it has been successfully applied in areas such as image analysis \cite{Carlsson2009}, genomics \cite{Nicolau2011}, materials science \cite{Hiraoka2016}, transportation \cite{Lam2018}, and chemistry \cite{Kriz2018}. The central goal is straightforward: extract stable, interpretable, and noise-robust structural information that summarizes the geometry and topology of the data across different resolutions.\\

    Even with its wide adoption, the use of TDA for graph-structured data remains relatively limited. Many real-world systems, such as social networks, communication networks, and blockchain transaction graphs, are naturally represented as graphs, but traditional graph-mining methods often focus only on nodes and edges and struggle to incorporate higher-order structures such as triangles or their evolution over time. TDA, in contrast, can naturally integrate node or edge attributes, capture multi-scale connectivity patterns, and track topological changes across nested subgraphs. These features give TDA a unique advantage in analyzing large and noisy networks. However, persistence diagrams, which is the main output of TDA, are not immediately usable in machine-learning models, which motivates the development of vectorized summaries such as persistence landscapes \cite{Bubenik2015}, persistence images \cite{Adams2017}, and Betti curves \cite{ChazalMichel2021,Chung2022} from a persistence diagram and then vectorize it to obtain a finite-dimensional vector.\\
    
    In \cite{Islambekov2024}, the Vector of Averaged Bettis (VAB) is introduced as a simple, scalable, and interpretable TDA summary tailored for cryptocurrency transaction graphs. VAB transformed persistence diagrams into fixed-length numerical vectors that could be combined seamlessly with standard graph features. Using Ethereum transaction data, it is shown showed that VAB provided meaningful predictive power for anomaly detection and helped reveal structural patterns associated with sudden price changes. This earlier study established one of the first TDA-based frameworks for analyzing blockchain graphs and demonstrated the value of topological summaries in financial network analysis.\\
    	
		In this paper, we take a different approach. Existing methods like VAB and Betti curves measure the cumulative presence of features, tracking which features are alive at each filtration scale and averaging or integrating this information. Persistence images capture the spatial distribution of feature lifespans in the birth-persistence plane. Instead, we measure the rate at which features appear and disappear and we call this rate-of-change perspective topological velocity \footnote{We note that the term 'topological velocity' appears in other scientific contexts with different meanings: in quantum physics, it refers to the velocity field of Bloch electrons and associated topological invariants \cite{Fan2022}; in computational biology, TopoVelo uses spatial graph topology to infer RNA velocity \cite{Gu2025}. Our usage is distinct, and we measure the rate of appearance and disappearance of topological features in persistence diagrams derived from network data.}. It works well for anomaly detection because anomalies in networks usually appear as sudden changes in structure, not slow, gradual shifts. We introduce the Overlap-Weighted Hierarchical Normalized Persistence Velocity (OW-HNPV), a novel persistence diagram summary that measures this velocity through overlap-based weighting of feature lifespans across multiple hierarchical scales. Unlike existing methods, OW-HNPV automatically downweights short-lived, potentially noisy features while capturing the temporal dynamics of network topology. The hierarchical construction produces a fixed-dimensional vector suitable for standard multivariate analysis and machine learning pipelines. Importantly, we establish theoretical stability guarantees for OW-HNPV with respect to the $L^1$ 1-Wasserstein distance on persistence diagrams, which is a significant advantage for comparing networks of varying complexity.\\

		The remainder of this paper is organized as follows. Section 2 reviews necessary background on topological data analysis, including persistence diagrams and existing vectorization methods. Section 3 introduces our Overlap-Weighted Hierarchical Normalized Persistence Velocity (OW-HNPV) summary, detailing its construction. Section 4 establishes theoretical stability guarantees for OW-HNPV with respect to the Wasserstein distance on persistence diagrams. Section 5 presents experimental results on simulated and real-world cryptocurrency transaction networks, demonstrating the effectiveness of OW-HNPV for anomaly detection. Section 6 concludes with a summary of contributions and directions for future research.
	
	\section{Background}
	
	Persistent homology is a mathematical tool used to study global shape and local geometry of data at different scales \cite{EdelsbrunnerHarer2008, ZomorodianCarlsson2005}. It starts with a series of simplicial complexes, $K_1 \subseteq K_2 \subseteq \cdots \subseteq K_m$ that grow step by step. As these shapes expand, persistent homology keeps track of when features like connected components, loops, or holes appear and when they disappear.\\
	
	A persistence diagram $PD$ in dimension $k$ is a multiset of points $(b, d) \in \R^2$ with $b \leq d$, where $b$ is the birth time and $d$ is the death time of a $k$-dimensional topological feature. The \emph{persistence} of the feature is $d - b$.\\
	
	The Wasserstein distance is standard for comparing persistence diagrams. 
	
	\begin{definition}
		Let $D_1$ and $D_2$ be two persistence diagrams. The $L^p$ $q$-Wasserstein distance \cite{KerberMorozovNigmetov2017}, $p,q\geq 1$, is given by
		$$d_{pq}(D_1, D_2) = \left( \inf_{\phi: D_1 \to D_2}  \sum_{u \in D_1} \|u - \phi(u)\|_p^q \right)^{1/q},$$
		where $\phi$ is a bijection such that each off-diagonal point is paired either with another off-diagonal point in the other diagram or with its orthogonal projection onto the diagonal.	
	\end{definition}

	Given a graph $G = (V, E, g)$ with node set $V$, edge set $E$ and a function $g: V \to \R$ assigning values to nodes. The lower-star filtration \cite{EdelsbrunnerHarer2010} constructs a nested sequence of simplicial complexes. We first enhance $G$ to include higher-dimensional simplices (triangles, tetrahedra) formed by the graph nodes and edges, then extend $g$ to all simplices by $g(\sigma) = \max_{v \in \sigma} g(v)$. The sublevel set at scale $t$ is:
	$$K_t = \{\sigma : g(\sigma) \leq t\}.$$
	
	This creates a filtration whose persistence diagram shows how the data’s topological features change as we raise the threshold $t$.\\
	
	In \cite{Islambekov2024}, a vector-based summary of a persistence diagram, called the Vector of Averaged Bettis (VAB), is introduced. It is obtained by integrating the associated Betti function \cite{ChazalMichel2021}, over consecutive scale values on a one-dimensional grid.

	\begin{definition}
		Let $D$ be a persistence diagram corresponding to homological dimension $k$. The Betti function, associated with $D$ and $k$, is given by
		\[
		\beta_k(t) = \sum_{(b,d) \in D} w(b, d)\,\chi_{[b,d)}(t),
		\]
		where $w : D \to \mathbb{R}$ is a weight function, and 
		$\chi_{[b,d)}(t)$ is 1 if $t \in [b, d)$ and 0 otherwise.
	\end{definition}
	
	Then, the Vector of Averaged Bettisd discretizes this function over intervals $[t_i, t_{i+1})$ and averages:
	$$\mathrm{VAB}_i = \frac{1}{\Delta t_i} \int_{t_i}^{t_{i+1}} \beta_k(s) \, ds,$$
	producing a vector $(\mathrm{VAB}_1, \ldots, \mathrm{VAB}_m) \in \R^m$.\\

    Another widely used functional summary is the persistence landscape \cite{Bubenik2015}, which converts a persistence diagram into a collection of piecewise-linear functions. Given a diagram $D$, the $k$-th landscape function $\lambda_k : \mathbb{R} \to [0,\infty)$ is defined as the $k$-th largest value of the triangular functions associated with each point $(b,d) \in D$. Persistence landscapes are stable with respect to the bottleneck distance and integrate smoothly with statistical and machine-learning methods. However, they emphasize feature prominence through persistence values and do not capture how quickly topological features appear or disappear across the filtration.\\
	
	Persistence images \cite{Adams2017} offer another widely used vectorization of persistence diagrams. Each diagram is first mapped into birth–persistence coordinates $(b, d-b)$, weighted by a chosen function (often based on persistence), and then smoothed using a Gaussian kernel over a discretized grid. This procedure yields a finite-dimensional, image-like representation that is stable and well suited for use in standard machine-learning workflows. While persistence images effectively capture the spatial distribution of topological features, they, like landscapes, emphasize static information rather than the temporal dynamics of feature appearance and disappearance that are central to our velocity-based approach.

    \section{Overlap-Weighted Hierarchical Normalized Persistence Velocity}
	
	Existing persistence diagram summaries capture the topological structure present at each filtration scale. The Vector of Averaged Bettis (VAB) quantifies the cumulative presence of topological features by averaging the Betti function over scale intervals, while persistence landscapes emphasize feature prominence through their persistence values. Persistence images provide an alternative approach by creating a continuous 2D representation through Gaussian smoothing of the (birth, persistence) plane. Despite their different constructions, these methods provide static characterizations of topological structure at each scale. 
	
	An alternative perspective, which we introduce in this paper, is to measure how rapidly the topology changes, that is, the rate at which topological features appear and disappear as we vary the filtration parameter. This velocity-based viewpoint offers several advantages for anomaly detection. Anomalies in time-varying networks usually appear as sudden changes in structure rather than slow buildup of features. Additionally, velocity-based summaries can admit stability properties that do not require persistence diagrams to have similar numbers of features.
	
	\subsection{Mathematical Framework}
	
	Mathematically, we formalize this intuition through the lens of measures on the real line. Given a persistence diagram $D = \{(b_i, d_i)\}_{i=1}^n$ of homological dimension $k$, we define counting measures for birth and death events:
	$$\mu_{\text{birth}} = \sum_{i=1}^{n} \delta_{b_i}, \quad 
	\mu_{\text{death}} = \sum_{i=1}^{n} \delta_{d_i},$$
	where $\delta_x$ denotes the Dirac measure at point $x$ (a point mass equal to 1 at $x$ and 0 elsewhere). For any interval $A \subseteq \mathbb{R}$, these measures count the number of births and deaths: $\mu_{\text{birth}}(A) = |\{i : b_i \in A\}|$ and $\mu_{\text{death}}(A) = |\{i : d_i \in A\}|$.
	
	To create a functional representation capturing topological dynamics across scales, we partition the filtration range $[\alpha, \beta]$ into $m$ \emph{main intervals} $I_j = [s_j, s_{j+1})$ for $j = 1, \ldots, m$, where $s_1 = \alpha$ and $s_{m+1} = \beta$. Each main interval $I_j$ is further subdivided into $n_{\text{sub}}$ \emph{subintervals} $[t_\ell^j, t_{\ell+1}^j)$ for $\ell = 1, \ldots, n_{\text{sub}}$, where $t_\ell^j = s_j + (\ell-1) \cdot \Delta t_\ell^j$ and $\Delta t_\ell^j = (s_{j+1} - s_j)/n_{\text{sub}}$.
	
	This hierarchical structure allows us to capture velocity at multiple resolutions while producing a fixed-dimensional vector summary.
	
	\subsection{Hierarchical Normalized Averaged Velocity (HNAV)}
	
	We begin with the simplest velocity-based approach, which treats all features equally.
	
	For each subinterval $[t_\ell^j, t_{\ell+1}^j)$, we compute the local velocity by counting births and deaths:
	$$V^{j,\ell} = \frac{\mu_{\text{birth}}([t_\ell^j, t_{\ell+1}^j)) + \mu_{\text{death}}([t_\ell^j, t_{\ell+1}^j))}{2\Delta t_\ell^j},$$
	where $\Delta t_\ell^j = t_{\ell+1}^j - t_\ell^j$. 
	
	We then average over subintervals within each main interval:
	$$V^j = \frac{1}{n_{\text{sub}}} \sum_{\ell=1}^{n_{\text{sub}}} V^{j,\ell}.$$
	
	Finally, normalizing by the total number of features $n_k$ yields the \emph{Hierarchical Normalized Averaged Velocity} (HNAV):
	$$\text{HNAV}_k = \left(\frac{V^1}{n_k}, \frac{V^2}{n_k}, \ldots, \frac{V^m}{n_k}\right) \in \mathbb{R}^m.$$
	
	This $m$-dimensional vector captures the temporal dynamics of topological change across different filtration scales. However, it treats all features equally, regardless of their persistence or significance.
	
	\subsection{Hierarchical Weighted Normalized Averaged Velocity (HWNAV)}
	
	Persistence diagrams computed from real-world data often contain many short-lived features that may represent noise rather than meaningful topological structure. To address this, we introduce persistence weighting, where each feature is weighted by its persistence value $(d_i - b_i)$. Features that persist longer are more likely to represent true topological signal rather than noise.
	
	
	Given a persistence diagram $D = \{(b_i, d_i)\}_{i=1}^n$, we define weighted counting measures:
	$$\mu_{\text{birth}}^w = \sum_{i=1}^{n} (d_i - b_i) \cdot \delta_{b_i}, \quad \mu_{\text{death}}^w = \sum_{i=1}^{n} (d_i - b_i) \cdot \delta_{d_i}.$$
	
	For each subinterval $[t_\ell^j, t_{\ell+1}^j)$, instead of counting features, we sum their persistence values:
	$$W_b^{j,\ell} = \sum_{i : b_i \in [t_\ell^j, t_{\ell+1}^j)} (d_i - b_i), \quad W_d^{j,\ell} = \sum_{i : d_i \in [t_\ell^j, t_{\ell+1}^j)} (d_i - b_i).$$
	
	The weighted velocity for subinterval $[t_\ell^j, t_{\ell+1}^j)$ is:
	$$V_w^{j,\ell} = \frac{W_b^{j,\ell} + W_d^{j,\ell}}{2\Delta t_\ell^j}.$$
	
	Averaging over subintervals within main interval $I_j$:
	$$V_w^j = \frac{1}{n_{\text{sub}}} \sum_{\ell=1}^{n_{\text{sub}}} V_w^{j,\ell} = \frac{1}{n_{\text{sub}}} \sum_{\ell=1}^{n_{\text{sub}}} \frac{W_b^{j,\ell} + W_d^{j,\ell}}{2\Delta t_\ell^j}.$$
	
	
	Rather than normalizing by the number of features $n_k$, which can vary significantly across different persistence diagrams, we normalize by the total persistence:
	$$P_k = \sum_{i=1}^{n} (d_i - b_i).$$
	
	This yields the \emph{Hierarchical Weighted Normalized Averaged Velocity} (HWNAV):
	$$\bar{V}_w^j = \frac{1}{P_k} V_w^j, \quad j = 1, \ldots, m,$$
	and
	$$\text{HWNAV}_k = (\bar{V}_w^1, \bar{V}_w^2, \ldots, \bar{V}_w^m) \in \mathbb{R}^m.$$
	
	While HWNAV accounts for feature significance through persistence weighting, it does not capture how much each feature actually contributes to the topology within each specific interval. A feature with large persistence may be born or die far outside an interval, yet still contributes to the velocity in that interval based solely on its birth or death event. This motivates our final refinement.
	
	\subsection{Overlap-Weighted Hierarchical Normalized Persistence Velocity (OW-HNPV)}
	
	Our main contribution is a velocity summary that weights features not just by their persistence, but by how much they actually overlap with each interval, directly measuring their contribution to the topology at that scale.
	
	\begin{definition}[Overlap-Weighted Hierarchical Normalized Persistence Velocity]
		\label{def:ow_hnpv}
		Let $D = \{(b_i, d_i)\}_{i=1}^n$ be a persistence diagram of homological dimension $k$, where $(b_i, d_i)$ represents the birth and death times of the $i$-th topological feature with $b_i < d_i$. Let $[\alpha, \beta]$ be the filtration range, partitioned into $m$ main intervals $[s_j, s_{j+1})$ for $j = 1, \ldots, m$, where $s_1 = \alpha$ and $s_{m+1} = \beta$. Each main interval is further subdivided into $n_{\text{sub}}$ equal subintervals.
		
		For feature $i$ and subinterval $[t_\ell^j, t_{\ell+1}^j)$, define the overlap weight:
		$$w_i^{j,\ell} = \text{length}\left([b_i, d_i) \cap [t_\ell^j, t_{\ell+1}^j)\right),$$
		computed as:
		$$w_i^{j,\ell} = \max\left\{0, \min\{d_i, t_{\ell+1}^j\} - \max\{b_i, t_\ell^j\}\right\}.$$
		
		The velocity in subinterval $[t_\ell^j, t_{\ell+1}^j)$ is:
		$$V^{j,\ell} = \frac{1}{\Delta t_\ell^j} \sum_{i=1}^n w_i^{j,\ell}.$$
		
		The average velocity in main interval $j$ is:
		$$V^j = \frac{1}{n_{\text{sub}}} \sum_{\ell=1}^{n_{\text{sub}}} V^{j,\ell}.$$
		
		The Overlap-Weighted Hierarchical Normalized Persistence Velocity (OW-HNPV) is the vector:
		$$\mathbf{H} = (H^1, \ldots, H^m) \in \mathbb{R}^m,$$
		where:
		$$H^j = \frac{1}{P(D)} \cdot V^j = \frac{1}{P(D)} \cdot \frac{1}{n_{\text{sub}}} \sum_{\ell=1}^{n_{\text{sub}}} V^{j,\ell},$$
		and $P(D) = \sum_{i=1}^n (d_i - b_i)$ is the total persistence of $D$.
	\end{definition}
	
	
	The overlap weighting in OW-HNPV offers several advantages over the simpler variants. Unlike HWNAV, which assigns a feature's full persistence weight to its birth and death times, OW-HNPV distributes a feature's contribution across all intervals it overlaps. This more accurately reflects how features contribute to topology at different scales. Features with longer persistence naturally contribute more through larger overlaps, without needing separate weighting. A feature that exists throughout an interval contributes its full overlap length, while a feature barely touching the interval contributes proportionally less. Short-lived features (small $d_i - b_i$) have small overlaps with most intervals, automatically decreasing their contribution. This provides inherent noise filtering without thresholding. In the next section, we establish theoretical stability guarantees for OW-HNPV with respect to the $L^1_1$-Wasserstein distance on persistence diagrams, ensuring that small perturbations to the input network produce bounded changes in the summary vector.\\
	
	For a persistence diagram with $n$ features over $m$ main intervals with $n_{\text{sub}}$ subintervals, the computational complexity of OW-HNPV is $O(n \cdot m \cdot n_{\text{sub}})$ since computational complexity of overlap wights is $O(n \cdot m \cdot n_{\text{sub}})$ and computational complexity of velocities and averaging is $O(m \cdot n_{\text{sub}})$.\\
	
	In the following example, we complete the HNAV, HWNAV and OW-HNPV for a simple example.
	
	\begin{example}
		Consider a persistence diagram with three features representing different topological structures:
		\begin{center}
			Feature 1: $(b_1, d_1) = (0.1, 0.4)$, persistence = 0.3\\
			Feature 2: $(b_2, d_2) = (0.2, 0.8)$, persistence = 0.6\\
			Feature 3: $(b_3, d_3) = (0.65, 0.75)$, persistence = 0.1
		\end{center}
		
		The total persistence is $P(D) = 0.3 + 0.6 + 0.1 = 1.0$. We use the filtration range $[0, 1]$ with two main intervals $[0, 0.5)$ and $[0.5, 1.0)$, each with one subinterval ($n_{\text{sub}} = 1$).\\
		
		HNAV: We count births and deaths in each interval, and normalize by $n_k = 3$.
		\begin{itemize}
			\item In interval $[0, 0.5)$, we have 2 births and 1 death. Then the count is 3.
			\item In interval $[0.5, 1.0)$, we have 1 birth, 2 deaths. Then, the count is 3.
			\item The velocities are $V^{1} = 3/(2 \cdot 0.5) = 3.0$, and $V^{2} = 3/(2 \cdot 0.5) = 3.0$.
			\item Next, using normalizing, we obtain HNAV $= (3.0/3, 3.0/3) = (1.0, 1.0)$.
			
		\end{itemize}

		HWNAV: We weight by persistence at birth/death events, and normalize by $P(D) = 1.0$.
		\begin{itemize}
			\item In interval $[0, 0.5)$, we have 
			\begin{itemize}
				\item Birth weights: $0.3 + 0.6 = 0.9$
				\item Death weight: $0.3$
				\item Then, the total weights is $0.9 + 0.3 = 1.2$.
			\end{itemize}
			\item In interval $[0.5, 1.0)$, we have
			\begin{itemize}
				\item Birth weight: $0.1$
				\item Death weights: $0.6 + 0.1 = 0.7$
				\item Then, the total weights is $0.1 + 0.7 = 0.8$.
			\end{itemize}
			\item The velocities are $V_w^{1} = 1.2/(2 \cdot 0.5) = 1.2$, $V_w^{2} = 0.8/(2 \cdot 0.5) = 0.8$.
			\item Hence, by normalizing, we have HWNAV $= (1.2, 0.8)$.
		\end{itemize}
		
		OW-HNPV: We weight by actual overlap with each interval, and normalize by $P(D) = 1.0$.
		\begin{itemize}
			\item The feature overlaps with $[0, 0.5)$ are
			\begin{itemize}
				\item $w_1 = \min\{0.4, 0.5\} - \max\{0.1, 0\} = 0.4 - 0.1 = 0.3$
				\item $w_2 = \min\{0.8, 0.5\} - \max\{0.2, 0\} = 0.5 - 0.2 = 0.3$
				\item $w_3 = 0$ (entirely outside)
				\item Then the total weight is $0.3 + 0.3 = 0.6$.
			\end{itemize}
			\item The feature overlaps with $[0.5, 1.0)$ are
			\begin{itemize}
				\item $w_1 = 0$ (entirely outside)
				\item $w_2 = \min\{0.8, 1.0\} - \max\{0.2, 0.5\} = 0.8 - 0.5 = 0.3$
				\item $w_3 = \min\{0.75, 1.0\} - \max\{0.65, 0.5\} = 0.75 - 0.65 = 0.1$
				\item Then the total weight is $0.3 + 0.1 = 0.4$.
			\end{itemize}
			\item The velocities are $V^{1} = 0.6/0.5 = 1.2$, $V^{2} = 0.4/0.5 = 0.8$.
			\item Hence, by normalizing, we have OW-HNPV $= (1.2, 0.8)$.
		\end{itemize}

	\end{example}

	\section{Stability Analysis}
	A fundamental property of any topological summary is stability that is small changes to the input data should result in small changes to the summary. For persistence diagrams, stability is typically measured using the Wasserstein distance between diagrams.\\
	
	In this section, we establish a stability theorem for OW-HNPV, showing that it is Lipschitz continuous 
    with respect to the $L^1$ 1-Wasserstein distance on persistence diagrams. Our result does not require the persistence diagrams to have the same number of features, which is a significant advantage over the unweighted variant.\\
	
	For the stability analysis, we denote the Overlap-Weighted Hierarchical Normalized Persistence Velocit as $\mathbf{H} = (H^1, \ldots, H^m) \in \mathbb{R}^m$, where:
	$$H^j =  \frac{1}{P(D)} \cdot \frac{1}{n_{\text{sub}}} \sum_{\ell=1}^{n_{\text{sub}}} \left (\frac{1}{\Delta t_\ell^j} \sum_{i=1}^n w_i^{j,\ell}\right),$$
	where $P(D) = \sum_{i=1}^{n} (d_i - b_i)$ and $n_{\text{sub}}$ is the number of subsubintervals per main interval.

	\subsection{Key lemmas}
	We first establish key lemmas that will be used in the proof of the main stability theorem.
    
	\begin{lemma}[Total Overlap Property]
		\label{lem:total_overlap}
		For any feature $(b_i, d_i)$ and a partition of $[\alpha, \beta]$ into 
		subintervals $\{I_\ell\}_{\ell=1}^{m \cdot n_{\text{sub}}}$ (where $m$ is 
the number of main intervals and each main interval is subdivided into $n_{\text{sub}}$ subintervals),, the sum of overlaps equals the persistence:
		$$\sum_{\ell=1}^{m \cdot n_{\text{sub}}} w_i^\ell \leq d_i - b_i,$$
		where $w_i^\ell = \text{length}([b_i, d_i) \cap I_\ell)$.
	\end{lemma}
	
	\begin{proof}
		Since the intervals $\{I_\ell\}$ partition $[\alpha, \beta]$ and are disjoint:
		\begin{align*}
			\sum_{\ell=1}^{m \cdot n_{\text{sub}}} w_i^\ell 
			&= \sum_{\ell=1}^{m \cdot n_{\text{sub}}} \text{length}([b_i, d_i) \cap I_\ell) \\
			&= \text{length}\left([b_i, d_i) \cap \bigcup_{\ell=1}^{m \cdot n_{\text{sub}}} I_\ell\right) \\
			&= \text{length}([b_i, d_i) \cap [\alpha, \beta]) \\
			&\leq d_i - b_i.
		\end{align*}
		
		\end{proof}
	
	\begin{lemma}[Overlap Difference Bound]
		\label{lem:overlap_difference}
		Let $(b, d)$ and $(b', d')$ be two features (or one feature matched to 
		another) with:
		$$\|(b, d) - (b', d')\|_1 = |b - b'| + |d - d'| = \delta.$$
		
		For any interval $I = [t_1, t_2)$, let:
		\begin{align*}
			w &= \text{length}([b, d) \cap I), \\
			w' &= \text{length}([b', d') \cap I).
		\end{align*}
		
		Then:
		$$|w - w'| \leq 2\delta.$$
	\end{lemma}
	
	\begin{proof}
		The overlap weights can be computed
		$$
			w = \max\{0, \min\{d, t_2\} - \max\{b, t_1\}\}, \quad
			w' = \max\{0, \min\{d', t_2\} - \max\{b', t_1\}\}.
		$$
		
		Let 
        \vspace{-.2in}
		\begin{align*}
			\bar{d} &= \min\{d, t_2\}, \quad \bar{d}' = \min\{d', t_2\}, \\
			\underline{b} &= \max\{b, t_1\}, \quad \underline{b}' = \max\{b', t_1\}.
		\end{align*}
		
		Then, $w = \max\{0, \bar{d} - \underline{b}\}$ and 
		$w' = \max\{0, \bar{d}' - \underline{b}'\}$.
		
		\textbf{Case 1:} Both intervals have positive overlap, i.e. $w > 0$ and $w' > 0$.
		
		Then,
		$$
			|w - w'| = |(\bar{d} - \underline{b}) - (\bar{d}' - \underline{b}')| = |(\bar{d} - \bar{d}') - (\underline{b} - \underline{b}')| \leq |\bar{d} - \bar{d}'| + |\underline{b} - \underline{b}'|.
		$$
		
		We have,
		\begin{align*}
			|\bar{d} - \bar{d}'| = |\min\{d, t_2\} - \min\{d', t_2\}| \leq |d - d'|, \\
			|\underline{b} - \underline{b}'| = |\max\{b, t_1\} - \max\{b', t_1\}| \leq |b - b'|.
		\end{align*}
		
		Therefore,
		$$|w - w'| \leq |d - d'| + |b - b'| = \|(b,d) - (b',d')\|_1 = \delta.$$
		
		\textbf{Case 2:} $w = 0$ but $w' > 0$ (or vice versa). 
		
		This means $[b, d) \cap I = \emptyset$ but $[b', d') \cap I \neq \emptyset$.
		
		Subcase 2a: $d \leq t_1$ if feature ends before interval starts. Then, $\bar{d} = d$ and $\underline{b} \geq t_1$, so $w = \max\{0, d - t_1\} = 0$.
		
		For $w' > 0$, we need $d' > t_1$. Since $|d - d'| \leq \delta$,
		$d' \leq d + \delta .$ Also, since $|b - b'| \leq \delta$ $b' \leq b + \delta$. Then, for the overlap, we have:
%
		\begin{align*}
		w' = \bar{d}' - \underline{b}' = \min\{d', t_2\} - \max\{b', t_1\} \leq  \min\{d', t_2\} - b' \quad (\text{suppose } b' \geq t_1) \\ \leq d' - b' = (d' - d) + (d - b') \leq (d' - d) + (d - b) \leq (d' - d) + (b' - b) \leq 2\delta.
		\end{align*}
		
		Subcase 2b: $b \geq t_2$ if feature starts after interval ends.
		
		By symmetry: $w' \leq 2\delta$.	Therefore: $|w - w'| = |0 - w'| = w' \leq 2\delta$.\\
		
		\textbf{Case 3:} Both $w = 0$ and $w' = 0$.
		
		Then $|w - w'| = 0 \leq 2\delta$.
		
		Combining all cases, we have $|w - w'| \leq 2\delta$ which is the desired result.
		\end{proof}
	
	\begin{lemma}[Velocity Difference Bound]
		\label{lem:velocity_difference}
		Let $D_1 = \{(b_i^{(1)}, d_i^{(1)})\}_{i=1}^n$ and 
		$D_2 = \{(b_i^{(2)}, d_i^{(2)})\}_{i=1}^n$ be two persistence diagrams 
		with $n$ features each, and let $\gamma$ be an optimal bijection with 
		respect to $d_{11}(D_1, D_2)$.
		
		For any subinterval $I = [t_\ell, t_{\ell+1})$ with width $\Delta t = t_{\ell+1} - t_\ell$:
		$$\left|V^{(1)}(I) - V^{(2)}(I)\right| \leq \frac{2 \cdot d_{11}(D_1, D_2)}{\Delta t},$$
		where $V^{(k)}(I) = \frac{1}{\Delta t} \sum_{i=1}^n w_i^{(k)}$ and 
		$w_i^{(k)} = \text{length}([b_i^{(k)}, d_i^{(k)}) \cap I)$.
	\end{lemma}
	
	\begin{proof}
		By definition of $d_{11}$ and the optimal matching $\gamma$, we have
		$$d_{11}(D_1, D_2) = \sum_{i=1}^n \left[|b_i^{(1)} - b_{i}^{(2)}| + |d_i^{(1)} - d_{i}^{(2)}|\right],$$
		where we relabel by supposing $\gamma$ is the identity (feature $i$ in $D_1$ matches to 
		feature $i$ in $D_2$).
		
		The velocity difference is
		\begin{align*}
			\left|V^{(1)}(I) - V^{(2)}(I)\right| 
			&= \left|\frac{1}{\Delta t} \sum_{i=1}^n w_i^{(1)} - \frac{1}{\Delta t} \sum_{i=1}^n w_i^{(2)}\right| \\
			&= \frac{1}{\Delta t} \left|\sum_{i=1}^n \left(w_i^{(1)} - w_i^{(2)}\right)\right| \\
			&\leq \frac{1}{\Delta t} \sum_{i=1}^n \left|w_i^{(1)} - w_i^{(2)}\right|.
		\end{align*}
		
		By Lemma~\ref{lem:overlap_difference}, for each feature $i$, we have
		$$\left|w_i^{(1)} - w_i^{(2)}\right| \leq 2 \left(|b_i^{(1)} - b_i^{(2)}| + |d_i^{(1)} - d_i^{(2)}|\right).$$
		
		Therefore,
		\begin{align*}
			\left|V^{(1)}(I) - V^{(2)}(I)\right| 
			&\leq \frac{1}{\Delta t} \sum_{i=1}^n 2\left(|b_i^{(1)} - b_i^{(2)}| + |d_i^{(1)} - d_i^{(2)}|\right) \\
			&= \frac{2}{\Delta t} \sum_{i=1}^n \left(|b_i^{(1)} - b_i^{(2)}| + |d_i^{(1)} - d_i^{(2)}|\right) \\
			&= \frac{2}{\Delta t} \cdot d_{11}(D_1, D_2).
		\end{align*}
		\end{proof}
	
	\begin{lemma}[Total Persistence Stability]
		\label{lem:total_persistence_stability}
		For any two persistence diagrams $D_1$ and $D_2$ with optimal matching $\gamma$ 
		with respect to $d_{11}$, we have
		$$|P(D_1) - P(D_2)| \leq d_{11}(D_1, D_2),$$
		where $P(D) = \sum_{i} (d_i - b_i)$ is the total persistence.
	\end{lemma}
	
	\begin{proof}
		We have using the optimal bijection $\gamma$ (relabeled as identity)
		\begin{align*}
			|P(D_1) - P(D_2)| 
			&= \left|\sum_{i=1}^n (d_i^{(1)} - b_i^{(1)}) - \sum_{i=1}^n (d_i^{(2)} - b_i^{(2)})\right| \\
			&= \left|\sum_{i=1}^n \left[(d_i^{(1)} - d_i^{(2)}) - (b_i^{(1)} - b_i^{(2)})\right]\right| \\
			&\leq \sum_{i=1}^n \left|d_i^{(1)} - d_i^{(2)}\right| + \sum_{i=1}^n \left|b_i^{(1)} - b_i^{(2)}\right| \\
			&= \sum_{i=1}^n \left(|b_i^{(1)} - b_i^{(2)}| + |d_i^{(1)} - d_i^{(2)}|\right) \\
			&= d_{11}(D_1, D_2).
		\end{align*}
		\end{proof}

	\subsection{Main Stability Theorem}
	We now state and prove the main stability result for OW-HNPV.
	
		
    

	


    \begin{theorem}[Stability of OW-HNPV - $L^\infty$ norm]
    \label{thm:stability_overlap_hwnav}
    Let $D_1$ and $D_2$ be two persistence diagrams of the same homological dimension with positive total persistence $P(D_1), P(D_2) > 0$. The diagrams may have different numbers of off-diagonal points. Let $H_1 = (H_{1,1}, \ldots, H_{1,m})$ and $H_2 = (H_{2,1}, \ldots, H_{2,m})$ be their respective OW-HNPV vectors, computed with $m$ main intervals and $n_{\text{sub}}$ subintervals per main interval over the filtration range $[\alpha, \beta]$. Then,
    $$\|H_1 - H_2\|_\infty \leq \frac{3 n_{\text{sub}} m}{(\beta - \alpha) \cdot \min\{P(D_1), P(D_2)\}} \cdot d_{11}(D_1, D_2).$$
    \end{theorem}
    
	\begin{proof}

        Let $n_1 = |D_1|$ and $n_2 = |D_2|$ denote the number of off-diagonal points 
        in each diagram. Without loss of generality, assume $n_1 \leq n_2$ and set 
        $n = n_2$. We augment $D_1$ by adding $n_2 - n_1$ diagonal points, which are features 
        with birth time equal to death time. These diagonal points have zero 
        persistence and contribute zero to all overlap weights and velocities.\\
        
        Let $\gamma: D_1 \to D_2$ be an optimal bijection with respect to $d_{11}(D_1, D_2)$ on the augmented diagrams. Under this bijection, we can relabel so that $\gamma$ is the identity mapping. Note that for diagonal points added to $D_1$, we have $(b_i^{(1)}, d_i^{(1)}) = (b_i^{(1)}, b_i^{(1)})$, 
        so $d_i^{(1)} - b_i^{(1)} = 0$ and $w_i^{1,j,\ell} = 0$ for all $j, \ell$. Then, the total persistence $P(D_1) = \sum_{i=1}^{n_1} (d_i^{(1)} - b_i^{(1)})$ 
        includes only the original off-diagonal points, as diagonal points contribute zero.  Similarly, all sums involving overlap weights automatically exclude 
        diagonal points since they have zero overlap with any interval.\\

        Therefore, For each main interval $j \in \{1, \ldots, m\}$, we have
        $$H^{1,j} = \frac{1}{P(D_1)} \cdot \frac{1}{n_{\text{sub}}} \sum_{\ell=1}^{n_{\text{sub}}} V^{1,j,\ell},$$
        where
        $$V^{1,j,\ell} = \frac{1}{\Delta t_j^\ell} \sum_{i=1}^{n} w_i^{1,j,\ell}.$$
        
        Note that the sum over $i = 1, \ldots, n$ includes augmented diagonal points, 
        but these contribute zero to the sum. The rest of the proof proceeds identically 
        to the case $n_1 = n_2$.\\

		Step 1: We consider the bound for a Single Component. We bound $|H^{1,j} - H^{2,j}|$ by adding and subtracting $\frac{1}{P(D_1)}$ 
		in the second term:
		\begin{align*}
			|H^{1,j} - H^{2,j}| 
			&= \left|\frac{1}{P(D_1)} \cdot \frac{1}{n_{\text{sub}}} \sum_{\ell=1}^{n_{\text{sub}}} V^{1,j,\ell} 
			- \frac{1}{P(D_2)} \cdot \frac{1}{n_{\text{sub}}} \sum_{\ell=1}^{n_{\text{sub}}} V^{2,j,\ell}\right| \\
			&\leq \underbrace{\left|\frac{1}{P(D_1)} \cdot \frac{1}{n_{\text{sub}}} 
				\sum_{\ell=1}^{n_{\text{sub}}} V^{1,j,\ell} 
				- \frac{1}{P(D_1)} \cdot \frac{1}{n_{\text{sub}}} \sum_{\ell=1}^{n_{\text{sub}}} V^{2,j,\ell}\right|}_{\text{Term I}} \\
			&\quad + \underbrace{\left|\frac{1}{P(D_1)} \cdot \frac{1}{n_{\text{sub}}} 
				\sum_{\ell=1}^{n_{\text{sub}}} V^{2,j,\ell} 
				- \frac{1}{P(D_2)} \cdot \frac{1}{n_{\text{sub}}} \sum_{\ell=1}^{n_{\text{sub}}} V^{2,j,\ell}\right|}_{\text{Term II}}.
		\end{align*}
		
		Now, we work on bounding the terms. For term I, we have
		
		\begin{align*}
			\text{Term I} 
			&= \frac{1}{P(D_1)} \cdot \frac{1}{n_{\text{sub}}} 
			\left|\sum_{\ell=1}^{n_{\text{sub}}} \left(V^{1,j,\ell} - V^{2,j,\ell}\right)\right| \\
			&\leq \frac{1}{P(D_1)} \cdot \frac{1}{n_{\text{sub}}} 
			\sum_{\ell=1}^{n_{\text{sub}}} \left|V^{1,j,\ell} - V^{2,j,\ell}\right|.
		\end{align*}
		
		By Lemma~\ref{lem:velocity_difference}, for each subinterval $\ell$, we have
		$$\left|V^{1,j,\ell} - V^{2,j,\ell}\right| \leq \frac{2 \cdot d_{11}(D_1, D_2)}{\Delta t_\ell^j}.$$
		
		Therefore,
		\begin{align*}
			\text{Term I} 
			&\leq \frac{1}{P(D_1)} \cdot \frac{1}{n_{\text{sub}}} 
			\sum_{\ell=1}^{n_{\text{sub}}} \frac{2 \cdot d_{11}(D_1, D_2)}{\Delta t_\ell^j} \\
			&= \frac{2 \cdot d_{11}(D_1, D_2)}{P(D_1) \cdot n_{\text{sub}}} 
			\sum_{\ell=1}^{n_{\text{sub}}} \frac{1}{\Delta t_\ell^j}.
		\end{align*}
		
		For subintervals, we have $\Delta t_\ell^j = (t_{j+1} - t_j)/n_{\text{sub}}$, 
		where $t_{j+1} - t_j = (\beta - \alpha)/m$. Thus,
		$$\sum_{\ell=1}^{n_{\text{sub}}} \frac{1}{\Delta t_\ell^j} 
		= n_{\text{sub}} \cdot \frac{n_{\text{sub}} \cdot m}{\beta - \alpha} 
		= \frac{n_{\text{sub}}^2 \cdot m}{\beta - \alpha}.$$
		
		Therefore,
		$$\text{Term I} \leq \frac{2 \cdot d_{11}(D_1, D_2)}{P(D_1) \cdot n_{\text{sub}}} 
		\cdot \frac{n_{\text{sub}}^2 \cdot m}{\beta - \alpha} 
		= \frac{2n_{\text{sub}} m \cdot d_{11}(D_1, D_2)}{P(D_1) (\beta - \alpha)}.$$
		
		For bounding Term II, we have
        
		\begin{align*}
			\text{Term II} 
			&= \left|\frac{1}{P(D_1)} - \frac{1}{P(D_2)}\right| 
			\cdot \frac{1}{n_{\text{sub}}} \sum_{\ell=1}^{n_{\text{sub}}} V^{2,j,\ell} \\
			&= \frac{|P(D_1) - P(D_2)|}{P(D_1) \cdot P(D_2)} 
			\cdot \frac{1}{n_{\text{sub}}} \sum_{\ell=1}^{n_{\text{sub}}} V^{2,j,\ell}.
		\end{align*}
		
		By Lemma~\ref{lem:total_persistence_stability}, we have
		$$|P(D_1) - P(D_2)| \leq d_{11}(D_1, D_2).$$
		
		For the velocity sum, note that
		$$V^{2,j,\ell} = \frac{1}{\Delta t_\ell^j} \sum_{i=1}^n w_i^{2,j,\ell}.$$

        By Lemma \ref{lem:total_overlap}, each feature $i$ satisfies
        $$\sum_{\ell=1}^{n_{\text{sub}}} w_i^{2,j,\ell} \leq d_i^{(2)} - b_i^{(2)},$$
        with equality when the feature's entire lifespan $[b_i^{(2)}, d_i^{(2)})$ 
        lies within the main interval $[s_j, s_{j+1})$. Therefore, summing over all 
        features (the "worst case" where all features are fully contained in 
        interval $j$):
        $$\sum_{\ell=1}^{n_{\text{sub}}} \sum_{i=1}^n w_i^{2,j,\ell} 
          \leq \sum_{i=1}^n (d_i^{(2)} - b_i^{(2)}) = P(D_2).$$

        Then, we have
		\begin{align*}
			\sum_{\ell=1}^{n_{\text{sub}}} V^{2,j,\ell} 
			&= \sum_{\ell=1}^{n_{\text{sub}}} \frac{1}{\Delta t_\ell^j} \sum_{i=1}^n w_i^{2,j,\ell} \\
			&\leq P(D_2) \cdot \sum_{\ell=1}^{n_{\text{sub}}} \frac{1}{\Delta t_\ell^j} \\
			&= P(D_2) \cdot \frac{n_{\text{sub}}^2 m}{\beta - \alpha}.
		\end{align*}

		Therefore,
		\begin{align*}
			\text{Term II} 
			&\leq \frac{d_{11}(D_1, D_2)}{P(D_1) \cdot P(D_2)} 
			\cdot \frac{1}{n_{\text{sub}}} \cdot P(D_2) \cdot \frac{n_{\text{sub}}^2 m}{\beta - \alpha} \\
			&= \frac{d_{11}(D_1, D_2)}{P(D_1)} \cdot \frac{n_{\text{sub}} m}{\beta - \alpha}.
		\end{align*}
		
		Now, we combine Terms I and II. We have
		\begin{align*}
			|H^{1,j} - H^{2,j}| 
			&\leq \text{Term I} + \text{Term II} \\
			&\leq \frac{2n_{\text{sub}} m \cdot d_{11}(D_1, D_2)}{P(D_1)(\beta - \alpha)} 
			+ \frac{n_{\text{sub}} m \cdot d_{11}(D_1, D_2)}{P(D_1)(\beta - \alpha)} \\
			&= \frac{3n_{\text{sub}} m \cdot d_{11}(D_1, D_2)}{P(D_1)(\beta - \alpha)}.
		\end{align*}
		
		By swapping roles of $D_1$ and $D_2$, we have
		$$|H^{1,j} - H^{2,j}| \leq \frac{3n_{\text{sub}} m \cdot d_{11}(D_1, D_2)}{P(D_2)(\beta - \alpha)}.$$
		
		By taking the minimum, we have
		$$|H^{1,j} - H^{2,j}| \leq \frac{3n_{\text{sub}} m \cdot d_{11}(D_1, D_2)}{(\beta - \alpha) \cdot \min\{P(D_1), P(D_2)\}}.$$
		
		Step 2: We consider the $L^\infty$ Norm on Output.	The $L^\infty$ norm on $\mathbb{R}^m$ is
		\begin{align*}
			\|\mathbf{H}^1 - \mathbf{H}^2\|_\infty 
			&= \max_{j=1,\ldots,m} |H^{1,j} - H^{2,j}| \\
			&\leq \frac{3n_{\text{sub}} m}{(\beta - \alpha) \cdot \min\{P(D_1), P(D_2)\}} 
			\cdot d_{11}(D_1, D_2).
		\end{align*}
		
		This completes the proof.
		\end{proof}

    \section{Experimental Results}

	Our experiments demonstrate the effectiveness of velocity-based topological summaries, particularly OW-HNPV, for anomalous price prediction in cryptocurrency transaction networks. We conduct a comprehensive sensitivity analysis on the Ethereum transaction network spanning May 2017 to May 2018, comprising approximately 10 million transactions across 31 tokens.\\ 
	
	Our experimental setup works as follows which is similar to setup in \cite{Islambekov2024}. For each day, we build a transaction network using the 250 most active nodes (\texttt{topRank} = 250) so the computation stays manageable while still capturing the key market participants. From each daily network, we compute persistent homology using a lower-star filtration on graphs where node weights are defined by the average transaction amounts. We extract topological features in dimensions $k \in \{0, 1, 2\}$, which correspond to connected components, loops, and voids, respectively.\\
	
	Our goal is to predict large price jumps-days when the return moves more than $\pm 5\%$. We test how well the methods can predict these events across several time horizons $h \in \{1, 2, \ldots, 7\}$ days ahead. The value of $h$ tells us how far in advance we try to forecast the anomaly. Using multiple horizons helps us see whether topological features can provide early warning signals at different time scales.\\

	We compare six topological summary methods: three velocity-based approaches (HNAV, HWNAV, OW-HNPV) that capture the rate of topological change, and three established static summaries (VAB, PL, PI) that characterize persistence diagram structure at individual time points. All summaries are standardized to produce 30-dimensional feature vectors to ensure fair comparison. Performance is measured using AUC gains relative to a baseline model containing only graph-theoretic features (degree centrality, closeness centrality, betweenness centrality, and clustering coefficient).\\
	
	Three model configurations are evaluated such that Model M1 uses only baseline graph features; Model M2 augments the baseline with dimension-0 topological features and Model M3 includes baseline features plus both dimension-0 and dimension-1 topological features. Random forest classifiers with 500 trees are trained on each configuration, with performance assessed through 10-fold cross-validation repeated 10 times to ensure robustness. To analysis of the hierarchical structure parameters, for velocity-based methods, we vary the number of subintervals per main interval $n_{\text{sub}} \in \{1, 2, 3, 5, 10\}$ while fixing the number of main intervals at $m = 30$. This helps us examine how the choice of resolution affects the velocity estimates and the overall prediction results. Below we summarize the main findings.\\
	
	The Figure \ref{fig:sensitivity_main} shows AUC gains (\%) for all six methods across different values of $n_{\text{sub}}$ (number of subintervals per main interval) with fixed $m = 30$ main intervals. Results are shown for: (rows) Model configurations M2 (baseline + dimension-0 features) and M3 (baseline + dimensions 0 and 1); (columns) Prediction horizons $h \in \{1, 2, \ldots, 7\}$ days. Solid lines represent velocity methods (HNAV, HWNAV, OW-HNPV) that vary with $n_{\text{sub}}$, while dashed lines represent fixed methods (VAB, PL, PI) with constant performance.	
	\begin{figure}[H]
		\centering
		\includegraphics[width=\textwidth]{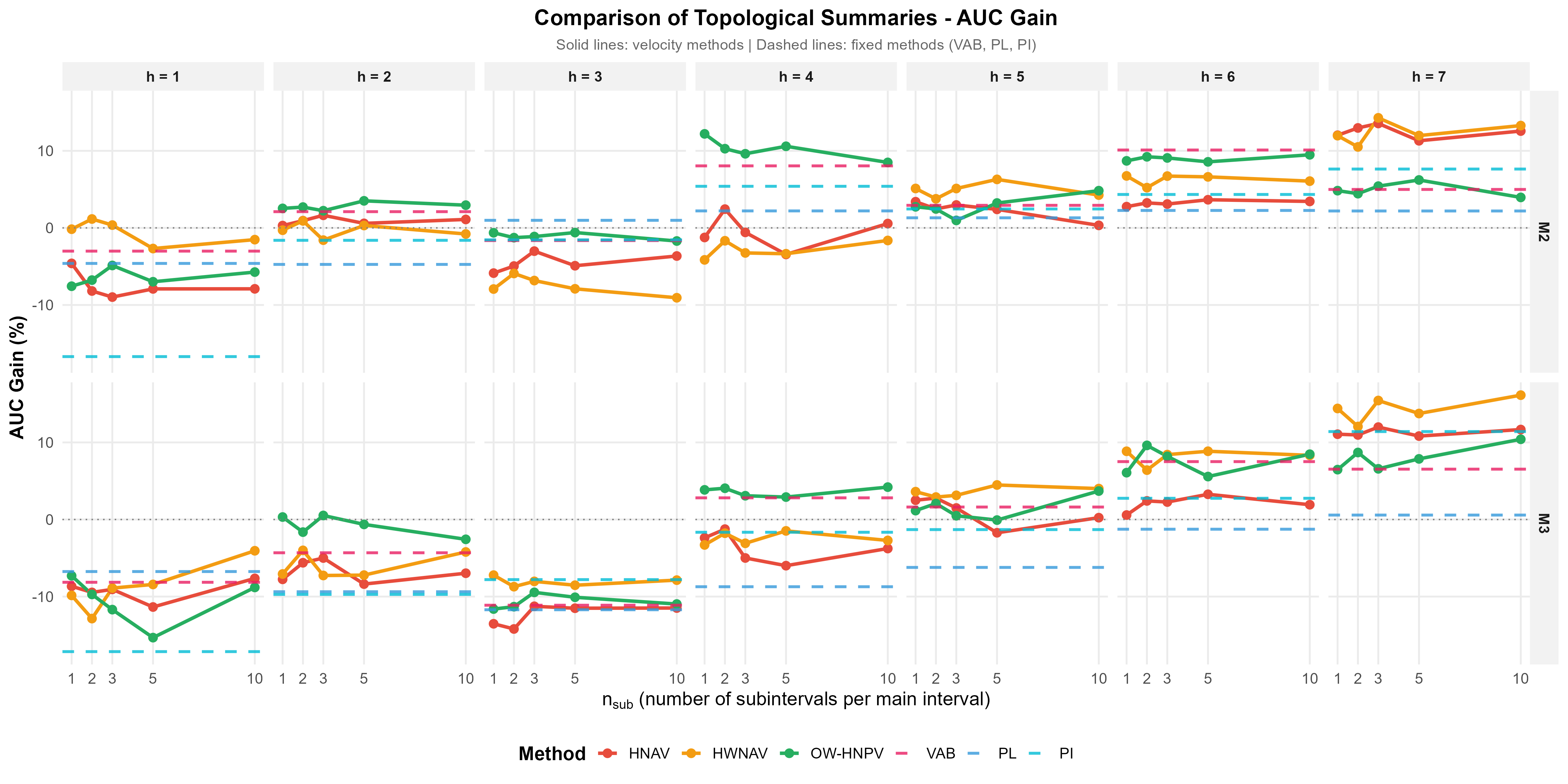}
		\caption{Sensitivity of Topological Summaries to Resolution Parameter. }
		\label{fig:sensitivity_main}
	\end{figure}

	Several clear patterns emerge from the comparison of topological summaries in Figure \ref{fig:sensitivity_main}. First, OW-HNPV (green) provides the most consistent and stable improvements over the fixed baselines, especially for intermediate and longer prediction horizons, and maintains competitive performance even when gains are modest. Second, AUC gains generally improve with increasing prediction horizon for $h \in {4,5,6,7}$ the velocity-based methods achieve noticeably higher gains than in the short-range settings $h \in {1,2,3}$. Third, Model M3 tends to outperform M2, indicating that combining both dimension-0 and dimension-1 topological features enhances predictive accuracy. Fourth, the velocity-based summaries exhibit sensitivity to the subinterval parameter $n_{\text{sub}}$, with performance often improving as $n_{\text{sub}}$ increases, stabilizing in the range $[3,10]$. Finally, the fixed summaries (VAB, PL, PI) appear as horizontal reference lines, and VAB typically provides the strongest static baseline, especially for larger values of $h$.\\

	To closer look at Model 3, we consider the optimal performance heatmap. In Figure \ref{fig:heatmap}, the heatmap displays the AUC gains (\%) for each method at their optimal $n_{\text{sub}}$ values across all prediction horizons. For velocity methods (HNAV, HWNAV, OW-HNPV), we selected the best-performing $n_{\text{sub}}$ for each horizon. Fixed methods (VAB, PL, PI) maintain constant values across $n_{\text{sub}}$.
	
	\begin{figure}[H]
		\centering
		\includegraphics[width=0.85\textwidth]{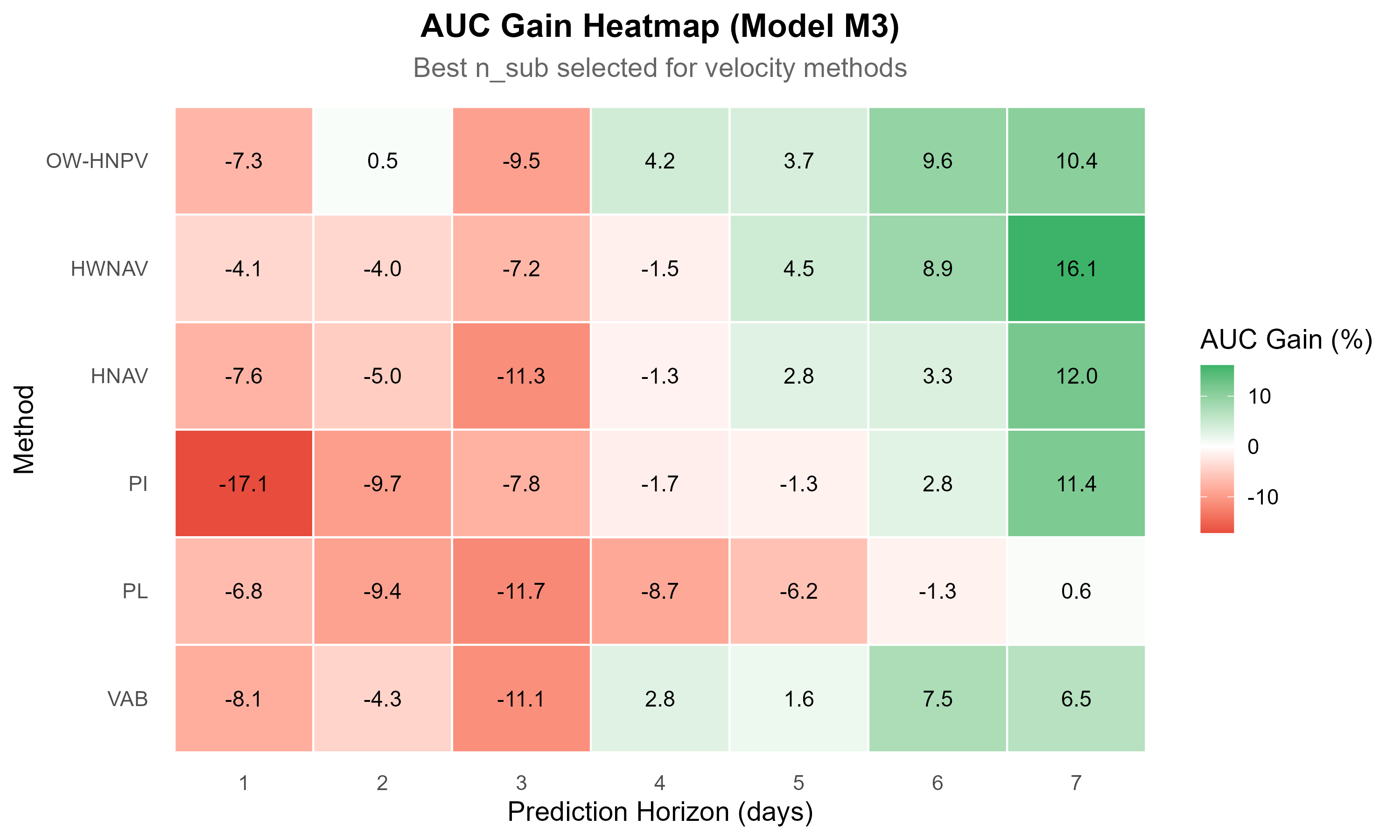}
		\caption{AUC Gain Heatmap for Best Configuration (Model M3). }
		\label{fig:heatmap}
	\end{figure}
	
	This heatmap for Model M3 reveals several notable trends. First, OW-HNPV achieves some of the strongest overall gains, reaching up to \textbf{10.4\%} improvement at a prediction horizon of $h = 7$ days. Similarly, HWNAV exhibits excellent long-range performance, with a peak gain of 16.1\% at $h = 7$, highlighting the benefits of persistence-based weighting. In contrast, short-term prediction remains challenging for horizons $h \in {1,2,3}$ all methods yield negative or only marginal improvements, indicating limited informative signal in early forecasts. Consistent with this, medium to long-term horizons ($h \geq 4$) show the most reliable positive gains across velocity-based and baseline methods. Among the fixed summaries, PI performs poorest for short-term predictions, with losses as large as –17.1\% at $h = 1$, while PL shows consistently negative gains across nearly all horizons, suggesting that persistence landscapes may not capture the temporal structure critical for this anomaly-detection task. Finally, VAB provides moderate and stable performance, becoming beneficial for $h \geq 4$ and reaching gains up to 7.5\% at $h = 6$, in line with its previously reported strengths in static topological summarization.\\

	In the following figure, we used bar graph to compare all six methods at the finest resolution tested ($n_{\text{sub}} = 10$) for three representative prediction horizons $h \in \{2, 4, 7\}$ days.
	
	\begin{figure}[H]
		\centering
		\includegraphics[width=0.9\textwidth]{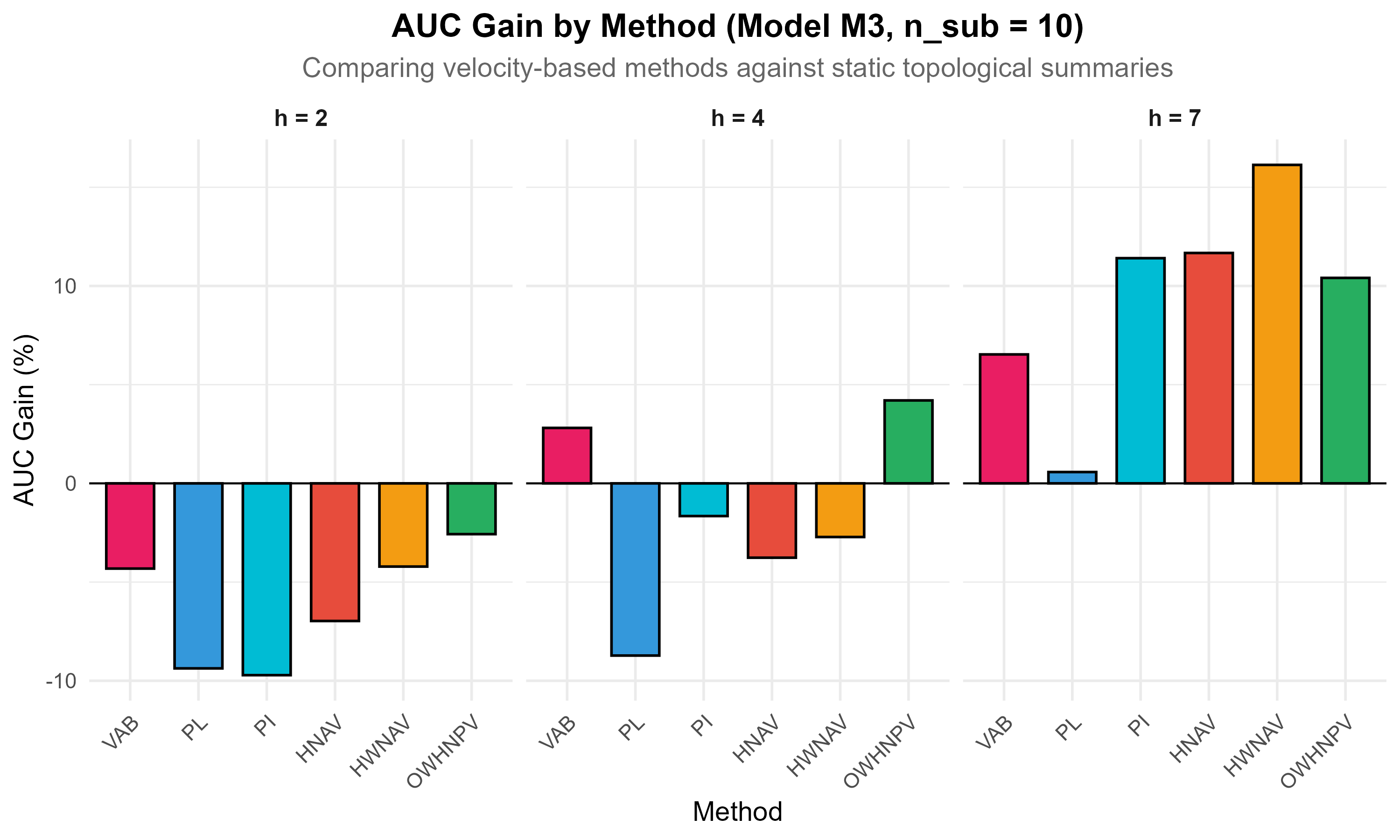}
		\caption{AUC Gains by Method at High Resolution ($n_{\text{sub}} = 10$, Model M3). }
		\label{fig:barplot}
	\end{figure}
	
	In Figure \ref{fig:barplot}, for the short-term horizon ($h = 2$), all methods exhibit negative AUC gains, reflecting the difficulty of very short-range anomaly prediction. Among them, OW-HNPV performs best, with a relatively small loss of about $-3\%$. HWNAV and VAB follow with similar declines around $-4\%$, while HNAV shows a larger decrease near $-8\%$. The static summaries PI and PL perform the worst, with losses approaching $-10\%$. This shows that velocity-based summaries do not provide an advantage at very short horizons, although OW-HNPV is still the most stable method.\\
	
	At the medium-term horizon ($h = 4$), results begin to improve, marking a transition point where temporal topological information becomes more predictive. OW-HNPV achieves the strongest gain at about $4\%$, while VAB also produces a small positive improvement. In contrast, PI and PL continue to show negative performance.\\
	
	For the long-term horizon ($h = 7$), nearly all methods show strong positive gains, demonstrating that longer prediction horizons benefit most from temporal topological structure. HWNAV is the top performer with a gain exceeding $15\%$, followed by HNAV at approximately $12\%$, PI at around $11\%$, and OW-HNPV at about $10\%$. Even VAB achieves a solid positive gain of roughly $6\%$, whereas PL remains near zero. These results highlight that velocity-based topological summaries, particularly HWNAV and OW-HNPV, excel when forecasting further into the future.

	\section{Conclusion}
	\label{sec:conclusion}
	
	This paper introduced the Overlap-Weighted Hierarchical Normalized Persistence Velocity (OW-HNPV) method, a novel topological data analysis approach for detecting anomalies in time-varying networks. Our theory shows that the method is stable and small changes in the input persistence diagrams lead to only small, bounded changes in the OW-HNPV summary. The size of this bound depends on the Wasserstein distance between diagrams and the chosen resolution levels.\\
	
	Our experiments on Ethereum transaction networks (May 2017--May 2018) show that velocity-based topological summaries are highly effective for predicting large price movements in cryptocurrency markets. HWNAV performs best for long-term predictions, reaching a 16.1\% gain at a seven-day horizon due to its emphasis on persistent topological features. OW-HNPV is the most stable and consistent method across all horizons. It performs especially well for medium-term forecasts, achieving a 4.2\% gain at the four-day horizon, and remains robust as the prediction window changes, making it particularly useful in practical applications. HNAV also achieves strong long-term results but exhibits more variability at shorter horizons, indicating a higher sensitivity to the prediction timeframe.\\
	
	VAB provides a consistent benchmark with modest gains for longer horizons. PI performs poorly in the short term but improves to an 11.4\% gain at seven days, suggesting that image-based features capture longer-range patterns. In contrast, PL underperforms across all horizons. Our results extend earlier VAB-based work by demonstrating how much additional information is gained by modeling the rate of topological change. For horizons $h \ge 4$, OW-HNPV outperforms VAB by up to 38\%, and HWNAV achieves more than a 100\% relative improvement, highlighting the importance of dynamic topological features for medium- and long-range forecasting. We also find that the hierarchical resolution parameter $n_{\text{sub}}$ significantly impacts performance, with effective values typically in the range $[3,10]$. Among the velocity-based methods, OW-HNPV is the least sensitive to this choice, confirming that overlap weighting improves stability. We should include that higher-dimensional features are important becasue Model M3, which incorporates both dimension-0 and dimension-1 features, consistently outperforms Model M2, showing that loop structures in transaction networks provide valuable predictive information.\\

	In conclusion, OW-HNPV offers an important step forward for topological analysis of temporal networks. It provides both theoretical stability and strong practical performance. By focusing on how topological structure changes over time rather than on static snapshots, our velocity-based approach opens new opportunities for studying complex systems where structural evolution is key to detecting and predicting anomalies.

    \section*{Acknowledgments}
    This research was developed by the author while teaching Topics in Graph Machine Learning (STAT 391) at the University of Evansville during Fall 2025. The author used Claude (Anthropic) to assist with manuscript editing and code debugging. All intellectual contributions are the author's own work.

\end{document}